\documentclass[conference]{IEEEtran}
\IEEEoverridecommandlockouts
\usepackage{cite}
\usepackage{amsmath,amssymb,amsfonts}
\usepackage{algorithmic}
\usepackage{graphicx}
\usepackage{textcomp}
\usepackage{xcolor}
\usepackage{algorithm}
\usepackage{algorithmic}
\usepackage{multirow}
\usepackage{amsthm}
\usepackage{soul}
\newtheorem{definition}{Definition}
\newtheorem{theorem}{Theorem}

\newtheorem{proposition}[theorem]{Proposition}

\usepackage{subcaption}
\def\BibTeX{{\rm B\kern-.05em{\sc i\kern-.025em b}\kern-.08em
    T\kern-.1667em\lower.7ex\hbox{E}\kern-.125emX}}
\begin{document}

\title{Learning Fair Classifiers via Min-Max F-divergence Regularization\\
\thanks{This work was supported by NSF grants CCF 2100013,  CAREER 1651492, CNS 2209951 and CNS 2317192.}}

\author{\IEEEauthorblockA{Meiyu Zhong ~~ Ravi Tandon}
\IEEEauthorblockA{Department of Electrical and Computer Engineering \\
University of Arizona, Tucson, USA \\
E-mail: \textit{\{meiyuzhong, tandonr\}}@arizona.edu
}}

\maketitle

\begin{abstract} 
As machine learning (ML) based systems are adopted in domains such as law enforcement, criminal justice, finance, hiring and admissions, ensuring the fairness of ML aided decision-making is becoming increasingly important. 
In this paper, we focus on the problem of fair classification, and introduce a novel min-max F-divergence regularization framework for learning fair classification models while preserving high accuracy. 

Our framework consists of two trainable networks, namely, a classifier network and a bias/fairness estimator network, where the fairness is measured using the statistical notion of F-divergence. We show that F-divergence measures possess convexity and differentiability properties, and their variational representation make them widely applicable in practical gradient based training methods. The proposed framework can be readily adapted to multiple sensitive attributes and for high dimensional datasets.  
We study the F-divergence based training paradigm for two types of group fairness constraints, namely, demographic parity and equalized odds. We present a comprehensive set of experiments for several real-world data sets arising in multiple domains (including COMPAS, Law Admissions, Adult Income, and CelebA datasets). 

To quantify the fairness-accuracy tradeoff, we introduce the notion of fairness-accuracy receiver operating characteristic (FA-ROC) and a corresponding \textit{low-bias} FA-ROC, which we argue is an appropriate measure to evaluate different classifiers. 
In comparison to several existing approaches for learning fair classifiers (including pre-processing, post-processing and other regularization methods), we show that the proposed F-divergence based framework achieves state-of-the-art performance with respect to the trade-off between accuracy and fairness.  
\end{abstract}


\section{Introduction}

Machine learning based solutions are being increasingly deployed and adopted in various sectors of society, such as criminal justice, law enforcement, hiring and admissions. Despite their impressive predictive performance, there is a large body of recent evidence \cite{mehrabi2021survey,zemel2013learning, zafar2017fairness} which shows a flip side of using data driven solutions: bias in decision making, which is often attributed to the inherent bias present in training data. For instance, in criminal justice, risk assessment algorithms are often used to assess the risk of recidivism (reoffence), which is then used together with human input for decision making \cite{angwin_larson_mattu_kirchner_2016}. A classic example is that of Correctional Offender Management Profiling for Alternative Sanctions (COMPAS) algorithm \cite{angwin_larson_mattu_kirchner_2016} which measures the recidivism risk and is used by judges for pretrial detention and release decisions. It was shown in \cite{angwin_larson_mattu_kirchner_2016} that COMPAS often falsely predicts a higher risk for some racial groups (specifically, african americans) compared to others. Another prominent example of bias with respect to gender is when a job advertisement system tends to show less STEM related job advertisements to women \cite{lambrecht2019algorithmic}. With the proliferation of data driven algorithms, ensuring fairness becomes crucial in the process of designing ML based decision making systems.

\noindent\textit{Notions of Fairness.} There is a vast literature on the notions of fairness \cite{dwork2012fairness,zafar2017fairness1,NIPS2017_a486cd07} which mainly falls into three categories: (1) Group Fairness \cite{barocas2016big,feldman2015certifying,hardt2016equality} which requires that the subjects in the protected and unprotected groups have equal probability of being assigned to the positive predicted class. (2) Individual Fairness \cite{dwork2012fairness,NIPS2017_a486cd07,yurochkin2019training} which requires that \textit{similar individuals} (measured by a domain specific similarity metric) should be treated similarly. (3) Causality-based Fairness \cite{kilbertus2017avoiding,NIPS2017_a486cd07}: which requires that using causality-based tools to design fair algorithms.
In this paper, we consider group fairness, and focus on the notions of demographic parity (DP) and equalized odds (EO).  The techniques for achieving group fairness can be mainly divided into (1) Pre-processing methods \cite{zemel2013learning}, which reduce bias by processing the training data before being used for training; (2) In-processing methods \cite{zafar2017fairness,cho2020fair1}, which add fairness constraints via regularization in the training process. (3) Post-processing methods \cite{pleiss2017fairness}, which appropriately modify the model parameters post training. 

\noindent\textbf{Related Work on Learning Fair Classifiers.} The prominent in-processing method for learning fair classifiers is via regularization methods, where the key idea is to train a classifier which minimizes the classification error regularized by a bias penalty (which measures the discrepancy of the classifier across population sub groups). We remark that there have been several other approaches on regularization based training for fair classification, which include: a) using a covariance proxy \cite{zafar2017fairness} to measure the bias between predictions and sensitive attributes. Unfortunately, ensuring small correlation does not necessarily satisfy the stronger requirement of statistical independence.; b) kernel density estimation (KDE) based methods which first estimate conditional probability distribution of classifier predictions for each population group and use these as a fairness regularizer  \cite{cho2020fair}. KDE methods are appropriate when the data dimensionality is relatively small but are not scalable for high dimensional problems; c) another approach is to balance the TPR and FPR (true- and false-positive rates) across population sub-groups while training \cite{bechavod2017learning}; and d) measuring bias by the mean of Hirschfeld-Gebelein-R{\'e}nyi (HGR) Maximum Correlation Coefficient \cite{mary2019fairness} or mutual information \cite{cho2020fair1} between predictions and sensitive attributes.

Admittedly, there are numerous choices for adding fairness constraints, and this opens up the following key questions:  a) What is the optimal choice of a fairness regularization for a given notion of fairness, as well as the dataset and sensitive attributes? b) How does the  regularization procedure impact the resulting tradeoff between accuracy and bias? c) How should one design a flexible procedure for learning fair classifiers which can work for high-dimensional datasets and is compatible with gradient based optimization? 

\noindent \textbf{Main Contributions.} To deal with the above challenges, we propose a general min-max $F$-divergence regularization framework for learning fair classifiers. $F$-divergence, denoted by $D_{f}(P||Q)$ measures the difference between two probability distributions $P, Q$ and different choices of the function $f(\cdot)$ lead to well-known divergence measures, such as KL divergence, Hellinger distance and Total Variation (TV) distance. Specifically, we propose to measure the bias using  $F$-divergence between the classifier probability distributions across protected and unprotected groups. We next summarize the main contributions of this paper.

\begin{itemize}
\vspace{-1pt}
    \item By leveraging the variational representation of $F$-divergence, we cast the training process as a min-max optimization problem, which is suitable for commonly used gradient based optimization methods. The flexibility of the framework is two-fold: a) it can be readily applied in high-dimensional datasets, and b) by varying the choice of $f$, one can test and validate different proxies of achieving fairness within a single rubric.    
    \vspace{-1pt}
\item To quantify the fairness-accuracy tradeoff, we introduce the notion of fairness-accuracy receiver operating characteristic (FA-ROC) and also provide some interesting theoretical properties when using Total Variation distance 
as the measure of bias. Within this context, we also introduce the notion of \textit{low-bias} FA-ROC, which we argue is an appropriate measure to evaluate different classifiers. 
\vspace{-1pt}
\item We present a comprehensive set of results on multiple real world datasets (namely, COMPAS, Adult Census, Law School admissions and CelebA), and show the superiority of the proposed approach versus existing regularization, pre- and post-processing methods as discussed above. As an example, for the Adult census dataset, $F$-divergence regularization leads to $\approx 13\%$ increase in FA-AUC (area under the curve) compared to the state-of-the-art regularization, pre- and post-processing methods for Demographic parity (we achieve a gain of $\approx 10\%$ in FA-AUC for Equalized odds). For the high dimensional dataset (CelebA), our method consistently achieves better performance and receives a gain of $6\%$ w.r.t EO constraints in the Low-bias region.
\end{itemize}

\section{Preliminaries on Fair Classification}\label{Sec:pre}
We consider a supervised classification problem, where we are given a dataset of $N$ users:
    ${\{X_n,\;Y_n,\;Z_n\}}_{n=1}^N$,
where $X_n$ denotes the set of features of user $n$;  $Y_n$ represents the true label of user $n$; $Z_n$ denotes the set of sensitive attributes of user $n$, which depends on the dataset and underlying context. For instance, in predicting recidivism risk in criminal justice, $X_n$ includes features such as prior criminal history, demographic information, charge (type of crime); $Z_n$ represents sensitive attributes, for instance, race or gender\footnote{More generally, $Z_n$ can take values from a discrete set, i.e., $Z_{n}\in \mathcal{G}$ and our formulation allows for the number of sensitive groups, i,e,. $|\mathcal{G}|\geq 2$ to be any finite number.}; $Y_n$ denotes ground truth like whether a user will re-offend in two years. Our goal is to build a fair binary classifier\footnote{We note that while our discussion in the paper is for binary classification, the proposed framework can be readily adopted for multi-class settings, as we discuss later in this section.}, which yields the estimate of the probability of the true label defined as follows:  
\begin{align}
\pi(\hat{Y}|X) \triangleq 
\begin{cases}
\pi(0|x) = P(\hat{Y}=0 \vert X=x)\\
\pi(1|x) = P(\hat{Y}=1 \vert X=x).
\end{cases}
\end{align}
Note that we do not use sensitive attributes $Z$ as an input to the classifier.  However, it is well known \cite{NIPS2017_a486cd07} that excluding the sensitive attributes alone does not necessarily lead to a fair classifier due to possible correlation between the sensitives attributes and features. For the scope of this paper, we focus on two statistical notions of group fairness: demographic parity (DP) and equalized odds (EO). Before we present group fairness notions, we first introduce the definition of F-divergence:
\begin{definition}(F-divergence)
Let function $f:{\mathbb{R}}_+\rightarrow\mathbb{R}$ be a convex, lower-semicontinuous function satisfying $ f(1) = 0 $. Given two probability distributions P and Q on a measurable space$(\mathcal X,\mathcal F)$, F-divergence is defined as: 
\begin{equation} \label{eq1}
    D_f(P\parallel Q)=E_Q\left[f\left(\frac{dP}{dQ}\right)\right] \nonumber
\end{equation}
\end{definition}
For instance, when $f(x)=x\log(x)$, this reduces to the KL divergence; $f(x)=(x-1)^2$ corresponds to $\chi^2$ divergence; $f(x)=(1-\sqrt x)^2$ corresponds to Squared Hellinger (SH) distance.
Next, we now define the notion of Demographic Parity (DP) as follows:
\begin{definition}(Demographic Parity) A binary classifier $\pi$ satisfies Demographic Parity (DP) if its prediction $\hat Y$ is independent of the sensitive attribute $Z$:
\begin{align}
\pi(1|Z=i)=\pi(1|Z=j)), ~~\forall i\neq j. \nonumber 
\end{align}
\end{definition}
Following from previous works \cite{cho2020fair,pleiss2017fairness,zafar2017fairness1}, the standard measurement of DP is the difference between the conditional output probability distribution of the classifier given sensitive group $i$ and $j$:
\begin{align}
\label{def:deltaDP}
    \Delta_{DP} :=  \sum_{i\neq j}|\pi(1 | Z=i) - \pi(1 | Z=j)|. 
\end{align}
\begin{figure*}[t]
\centering
\includegraphics[scale=0.2]{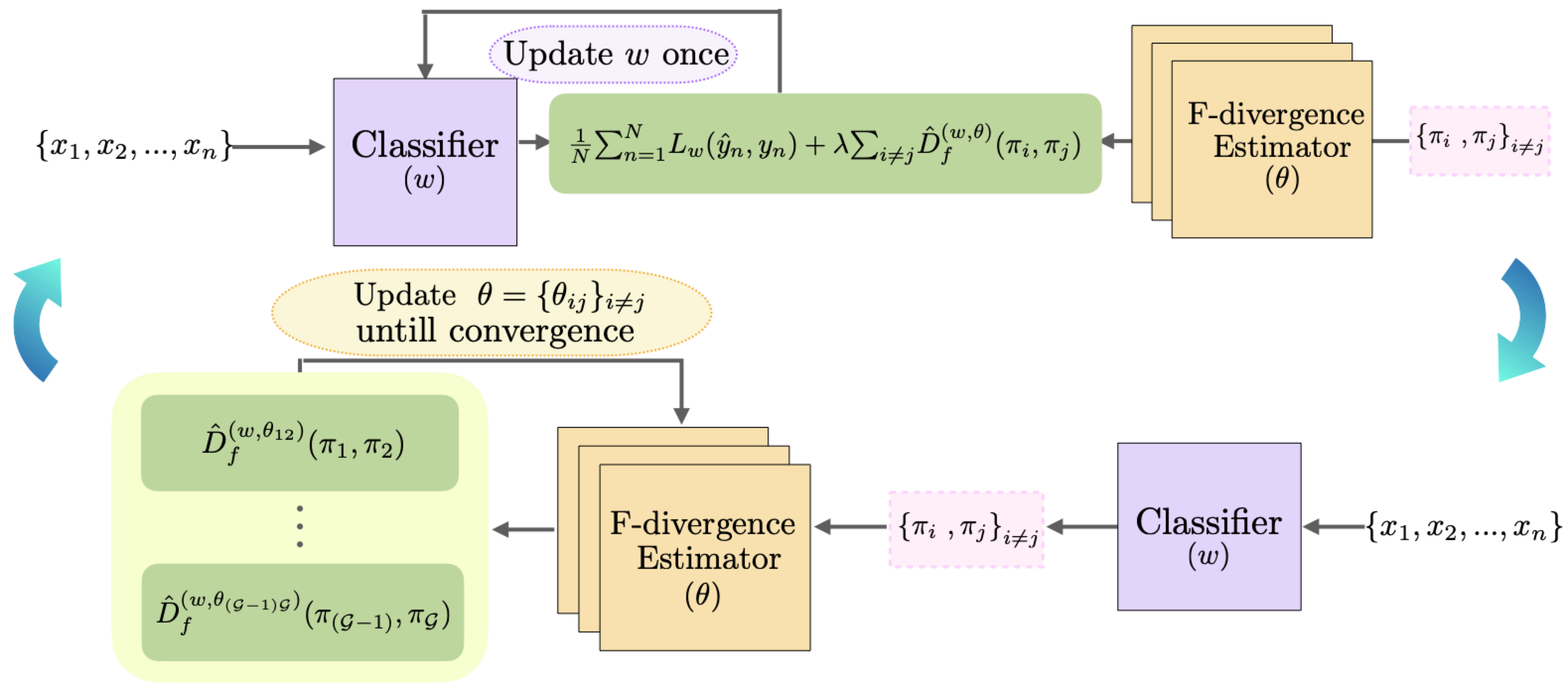}
\caption{Schematic of the Min-Max $F$-divergence regularization framework for fair training. The classifier (parameterized by $w$) is trained to minimize the regularized objective (containing classification loss $+$ F-divergence regularization term). The estimator (parameterized by $\theta$) estimates the F-divergence between distribution $\pi_i$ and distribution $\pi_j$, measuring the bias in classification across groups $(i,j)$. For DP, the distribution $\pi_i$ for a group $i$ is given by $\pi_i \sim \pi(\hat Y\left| Z=i\right.)$, whereas, for EO, $\pi_i \sim \pi(\hat Y\left| Z=i,Y=1\right)$. The two networks are trained alternatively, where for a fixed classifier, the $F$-divergence estimation is performed using a maximization problem leveraging the variational representation of $F$-divergence.}
\label{systemmodel}
\end{figure*}
Note that our notion of $\Delta_{DP}$ is the same as Total Variation distance. For the special case when $\Delta_{DP}=0$, this reduces to the notion of \textit{perfect} demographic parity \cite{barocas2016big}. We note that for the case of multi-class classification, the above notion generalizes by considering $\sum_{\hat y} \sum_{i\neq j} |\pi(\hat{Y} = \hat y|Z=i)-\pi(\hat{Y} = \hat y|Z=j)|$.
In a similar manner, we can define the notion of equalized odds as follows:
\begin{definition}(Equalized Odds)
A binary classifier $\pi$ satisfies Equalized Odds (EO) if its prediction $\hat Y$ is conditionally independent of its sensitive attribute $Z$ given the label $Y$.
\begin{align}
\pi(1|Z=i, Y=y) = \pi(1|Z=j, Y=y)), \nonumber 
\end{align}
where $\forall i \neq j~\text{and}~y \in \{0,1\}. $
Same as above, we define the standard measurement of EO as follows:
\begin{align}
    \hspace{-3pt}\Delta_{EO}\hspace{-3pt} :=\hspace{-5pt} \sum_{y}\sum_{i\neq j}|\pi(1 | Z=i, Y=y)\hspace{-2.5pt} -\hspace{-2pt} \pi( 1 | Z=j, Y=y)|.\label{def:deltaEO}
\end{align}
\end{definition}

\noindent \textbf{Motivation for $F$-divergence based Regularization}-- One prominent approach for learning fair classifiers is that of \textit{fairness regularization}, i.e., adding a penalty term in the training loss function, which acts as a proxy to capture the fairness constraints (either DP or EO). In this work, we propose to use $F$-Divergence (as defined above) as the fairness regularization term in the loss function. $F$-divergence family has natural benefits like convexity and differentiability, which makes it an ideal candidate for gradient based optimization algorithms. 
In addition, as we discussed in the introduction, compared to other approaches, such as correlation/covariance between sensitive attributes and classifier outputs, the $F$-divergence notions are stronger notions to capture dependence and provide stronger fairness guarantees. Furthermore, $F$-divergence notions often also have an operational interpretation;  for instance, the Kullback-Leibler (KL) and Chernoff divergences control the decay rates of error probabilities \cite{nguyen2010estimating}. As a member of the $F$-divergence family, prior work which uses mutual information \cite{cho2020fair1} between classifier predictions and sensitive attributes is therefore a special case of our framework.

\section{$F$-divergence Regularized Fair Training}\label{Sec:f-div}

We consider a classifier described by trainable parameters $w$, where $L_w(\hat{y}_n;y_n)$ is the loss function\footnote{For our experiments, we use binary cross-entropy loss function for training.} between the output of the classifier ($w$) and the ground truth of user $n$; and the fairness constrained learning can be formulated as the following optimization problem (denoted by \textbf{OPT}):
\begin{align}\label{fair_regu}
    &\underset w{min\;}\frac{1}{N}\sum_{n=1}^N L_w(\hat{y}_n;y_n)\; + \lambda \sum_{ i\neq j} D_f^{(w)}(\pi_i \parallel \pi_j),
\end{align}
where $D_f^{(w)}(\pi_i \parallel \pi_j)$ is the $F$-divergence between group $i$ and group $j$ for the classifier parameterized by $w$;  $\lambda$ is a hyperparameter that can be tuned to balance the trade-off between accuracy and fairness. Solving \textbf{OPT} requires the estimation of $F$-divergence for a classifier. To this end, we propose three $F$-divergence estimators to compare their performances on fair classification problem. We first leverage the variational representation of $F$-divergence which allows us to estimate F-divergence in an efficient manner. We show a schematic of the $F$-divergence based framework in Fig \ref{systemmodel}.

\subsection{Variational Representation of $F$-divergences (NN)}
It is well known \cite{nguyen2010estimating} that $F$-divergence between two distributions admits a variational representation, given as
\begin{align} 
    \hspace{-7pt}D_f(P\parallel Q) =\underset{T(\cdot)}{sup} ~E_{X\sim P}\left[T(X)\right] - E_{X\sim Q}\left[f^{*}(T(X))\right],\label{dp_f_div_variational}
\end{align}
where the function $f^\ast(t)=\underset{x\in dom_f}{sup}\{xt-f(x)\}$ denotes the convex conjugate (also known as the Fenchel conjugate) of the function $f$. The above variational representation involves a supremum over all possible functions $T(\cdot)$. We can obtain an estimate for $F$-divergence by replacing the supremum over a restricted class of functions. Specifically, if we use a parametric model $T_{\theta}$, (e.g., a neural network) with parameters $\theta$, then taking the supremum over the parameters $\theta$ yields a lower bound on $F$-divergence in \eqref{dp_f_div_variational} as stated next: 
\begin{align} 
    \hspace{-10pt}D_f(P\parallel Q) \hspace{-3pt}
\geq  \underset{\theta}{sup} \hspace{1pt}E_{X\sim P}\left[T_{\theta}(X)\right] - E_{X\sim Q}\left[f^{*}(T_{\theta}(X))\right].\label{dp_lower_bound} 
\end{align}
We use the above variational lower bound to estimate the $F$-divergence for fair classification as described next. 
For enforcing fairness constraints (DP/EO), we need to estimate $F$-divergence between joint distributions in different groups, i.e., 
$D_f(\pi_i\parallel \pi_j)~~ \forall ~ i\neq j, ~~ i, j \in \mathcal{G}$
where for Demographic parity (DP), $\pi_i \sim \pi(\hat Y\left| Z=i\right.)$, and for Equalized Odds (EO),  $\pi_i \sim \pi(\hat Y\left| Z=i, Y=1\right)$. The variational lower bound on $F$-divergence in \eqref{dp_lower_bound} can then be estimated as:
\begin{align}\label{dp_div}
    \underset{\theta}{\max} \frac{1}{M}\left(\sum_{m=1}^{M}\hspace{-1pt}T_\theta\left(x_{i}^{(m)}\right) - \sum_{m=1}^{M}\hspace{-1pt}f^\ast\left(T_\theta\left(x_{j}^{(m)})\right)\right)\right), 
\end{align}
where in \eqref{dp_div}, we have replaced the expectation operators with their empirical estimates, and $\{x_{i}^{(m)}\}$ denote i.i.d. samples drawn from the distribution $\pi_{i}$. Together with the above estimate, from \textbf{OPT} we arrive at the following min-max optimization problem (denoted by \textbf{MIN-MAX-OPT}):
\begin{align}
    &\underset w{\min\;}\underset\theta{\max\;}\frac{1}N\sum_{n=1}^N \mathcal L_w(\hat{y}_n;y_n)~\\
    &+\;\lambda\sum_{ i\neq j}\underbrace{\left(\frac{1}{M}\sum_{m=1}^{M}\left(T_\theta\left(x_i^{(m)}\right) - f^\ast\left(T_\theta\left(x_j^{(m)}\right)\right)\right)\right)}_{\triangleq \hat{D}^{(w,\theta)}_{f}(\pi_{i},\pi_{j})} \nonumber 
\end{align}
\noindent \textbf{MIN-MAX-OPT} can be solved by alternatively training the classifier $w$ and $F$-divergence estimator $\theta$. Specifically, we update the classifier weight $w$ while fixing the $F$-divergence estimator parameter $\theta^*$, and then we update $F$-divergence estimator parameter $\theta$ fixing the classifier weight $w^*$. The $F$-divergence estimator updates  each time as the classifier changes. We show the training details in Algorithm \ref{algo:f-div}. 

\begin{algorithm}[t]
\caption{$F$-divergence based Fair Training}\label{algo:f-div}
\begin{algorithmic}[1]
 \STATE \textbf{Input:} Sample $ \{x_{i}^{(1)},...,x_{i}^{(m)}\}$ from distribution $\pi_i$ and Sample $\{x_{j}^{(1)},...,x_{j}^{(m)}\}$ from distribution $ \pi_j$; 
 
\hspace{10mm}base classifier $w_t$; $F$-divergence estimator $\hat D_{f}$. 
    
 \STATE \textbf{Output:} Fair classifier ($w^*$) and corresponding $F$-divergence estimator ($\hat D_{f}^{*}$)

    \FOR{training iterations $(t=1, 2, \ldots, T_{1})$}{
	    \STATE \textbf{Divergence estimation for a fixed classifier $w_t$:} 
     Update divergence estimator $\theta$ for each pair  ($i,j$) of groups for $T_2$ steps (or until convergence) to maximize
	    \begin{center}
	    $
	    \hat{D}^{(w_t,\theta)}_{f}(\pi_{i},\pi_{j})
	    $ 
        \end{center}
        
	    \STATE \textbf{Classifier update:} Update classifier $w_t$ to minimize the regularized loss: 
	    \begin{center}
	    $
	    \frac{1}N\sum_{n=1}^N\mathcal L_{w_t}(\hat{y}_n;y_n)+\lambda\sum_{i \neq j}\hat{D}^{(w_{t},\theta_t)}_{f}(\pi _{i},\pi_{j})
	    $
	    \end{center}   
	}\ENDFOR

For demographic parity (DP): $\pi_i \sim \pi(\hat Y\left|Z_{i}\right.)$;
 for equalized odds (EO): $\pi_i \sim \pi(\hat Y\left|Z_{i}, Y=1\right)$.
\end{algorithmic}
\end{algorithm}

\begin{figure*}[t]
\centering
\includegraphics[scale=0.45]{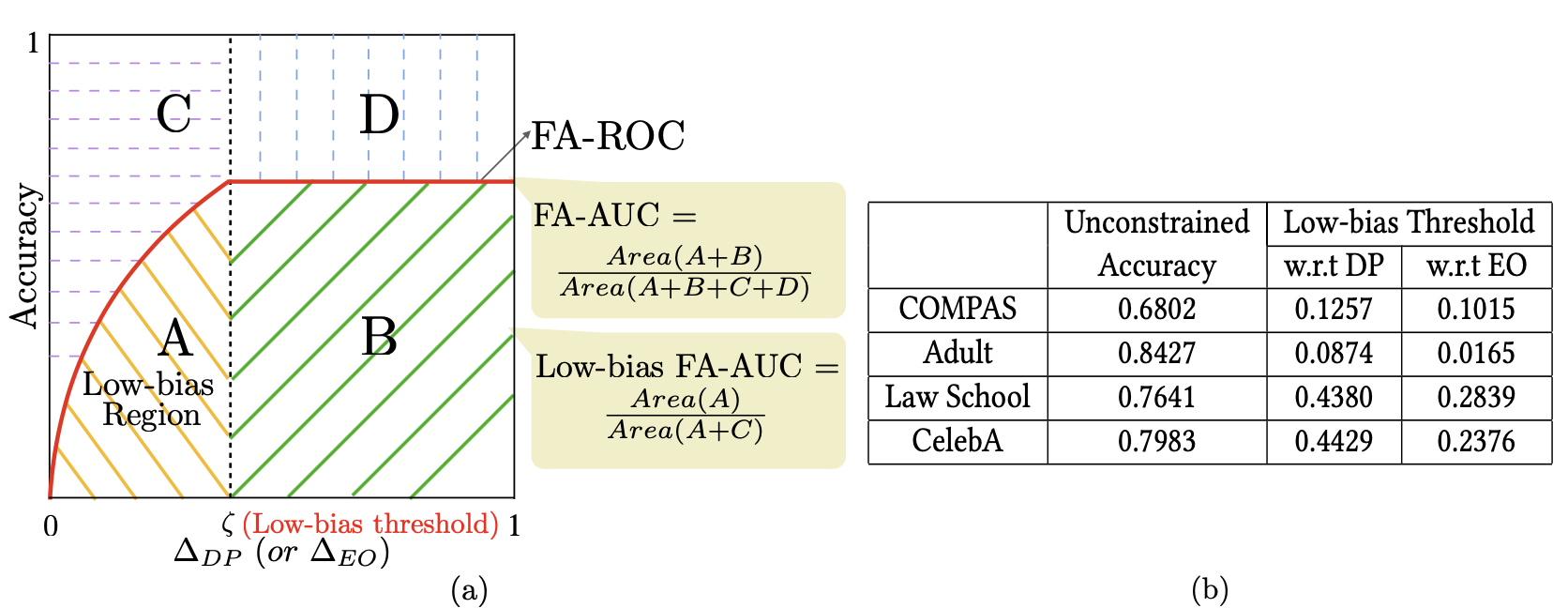}
\caption{(a) Quantification of Fairness-Accuracy (FA) tradeoffs via notions of FA-AUC and Low-Bias FA-AUC. The red curve (areas $A$ and $B$) represents the convex hull of the achievable fairness-accuracy pairs, i.e., $(\epsilon, a)$ obtained by varying the fairness regularization parameter. The \textit{low-bias} FA-AUC measures the AUC when the bias is less than a prescribed threshold, $\Delta_{DP}\leq \zeta$. 
We pick this low-bias threshold as the bias in the classifier when it is optimized without any fairness constraint. (b) Low-bias Threshold(s) w.r.t DP/EO and unconstrained test accuracy for four real world datasets. }
\label{Fig:auc}
\vspace{-13pt}
\end{figure*}
\section{Evaluation Metrics, Results and Discussion}\label{sec:experiment}
In this Section, we first discuss evaluation metrics (Section \ref{sec:evaluation}) to evaluate the tradeoffs between fairness and predictive test accuracy. Speficically, we propose the fairness-accuracy receiver operating characteristic (FA-ROC) and a corresponding low-bias version of FA-ROC as a quantitative measure of this tradeoff. We next present the details of our experimental setup in Section \ref{sec:exp_setup}, specifically, the datasets, details of classifier architectures and $F$-divergence estimators, training hyperparameters, as well as other existing approaches for learning fair classifiers which we compare against the proposed framework. Finally, in Section \ref{sec:exp_results},  we present and discuss the results and show the improved performance of the proposed $F$-divergence against various other existing approaches, including regularization, pre- and post-processing methods  for four real world datasets (COMPAS dataset \cite{angwin_larson_mattu_kirchner_2016} for recidivism prediction, Adult Census dataset \cite{dua_graff_2017} for income level prediction, Law School admissions dataset \cite{wightman_1998} for law students' scores prediction, CelebA dataset \cite{liu2015faceattributes} for facial image classification) and one synthetic dataset (Moon dataset \cite{pedregosa2011scikit}), which are widely used for the assessment of fairness in classification problems. The comprehensive additional experiments are provided in the Appendix (including multiple sensitive attributes, groupwise tradeoff(s) between fairness and accuracy, and comparison of different F-divergence estimators).
\vspace{-7pt}
\subsection{Fairness-Accuracy Tradeoff \& Evaluation Metrics} \label{sec:evaluation} 
In addition to imposing approximate fairness constraints, one is simultaneously interested in learning classifiers with high predictive accuracy. To this end, we next define the notion of fairness-vs-accuracy receiver operating characteristic, i.e., FA-ROC. For a classifier $\pi(\hat{Y}|X)$, we denote it's accuracy as $\text{Acc}(\pi)= \text{Pr}(\hat{Y} = Y)$.  

\begin{definition} A fairness-accuracy pair ($\epsilon, a$) w.r.t. demographic parity is achievable if there exists a classifier $\pi(\cdot)$ with $\text{Acc}(\pi)\geq a$ and $\Delta_{\text{DP}}(\pi)\leq \epsilon$. Similarly, one can define the achievability of $(\epsilon, a)$ w.r.t. equalized odds.  
\end{definition}
We next state a Proposition which reveals an interesting property of the achievable fairness-accuracy tuples. 
\begin{proposition}\label{Lemma: delta_acc}
Suppose  that the fairness-accuracy tuples $(\epsilon_1,a_1),(\epsilon_2,a_2),\ldots, (\epsilon_k,a_k)$ are achievable. Then any convex combination $(\sum_{i=1}^{k}\beta_i \epsilon_i, \sum_{i=1}^{k}\beta_i a_i)$, for $\beta_i \in [0,1], \sum_{i}\beta_i=1$ of these pairs is also achievable. Consequently, the convex hull of these pairs is also achievable.
\end{proposition}

\begin{proof}
Suppose that the fairness-accuracy pairs $(\epsilon_i, a_i), i=1,\ldots, k$ are achievable, i.e., there exists  classifiers $\pi_i, i =1,\ldots, k$ for which  $a_i = \pi_i(\hat Y = Y)$ and 
\begin{align}
    \epsilon_i =\Delta_{DP}(p_i,q_i) =|p_{i}-q_{i} |=|\pi_i(\hat Y | Z=0)-\pi_i(\hat Y | Z = 1)|\nonumber
\end{align}
We need to show that any convex combination of these tuples, i.e.,  $(\epsilon, a) = (\sum_{i}\beta_i \epsilon_i, \sum_{i}\beta_i a_i)$ is achievable.   To this end, we can construct a new classifier as follows: $\pi = \pi_i$ with prob. $\beta_i$, $i=1,2,\dots, k$, for $\beta_i \in [0,1], \sum_{i}\beta_i=1$.
We now prove that the above classifier achieves the desired $(\epsilon, a)$. 
The accuracy of $\pi$ can be lower bounded as 
\begin{align}
    \pi(\hat{Y} = Y) 
    &= \sum_{i=1}^{k} \beta_i \pi_i(\hat Y = Y) \geq \sum_{i=1}^{k}\beta_i a_i \triangleq a.
\end{align}
We next show the bound on the fairness constraint for the classifier $\pi$, where $\Delta_{DP}(p,q)$ can be upper bounded by: 
\begin{align} \label{eq:convexityf}
     \left|\sum_{i}\beta_i p_i -  \sum_{i}\beta_i q_i\right| \overset{(a)}{\leq} \sum_{i=1}^{k}\beta_i |p_i - q_i| \leq \sum_{i=1}^{k}\beta_i \epsilon_i \triangleq \epsilon \nonumber 
\end{align}
where (a) follows from the fact that the norm $\Delta_{DP}(p,q)$ is convex in the pair $(p,q)$ followed by Jensen's inequality. 
This completes the proof of Proposition \ref{Lemma: delta_acc}. 
\end{proof}

\begin{table*}[t]
\begin{center}
\scalebox{0.81}{
\begin{tabular}{|c|c|c|c|c|c|}
\hline
 Datasets &Moon & COMPAS &  Adult&  Law School & CelebA \\
\hline
 Learning Rate & 2e-3 / 2e-3  & 6e-4 / 6e-4 & 1e-2 / 1e-2  & 1e-4 / 1e-4 & 1e-3 / 1e-3 \\
\hline
Batch Size& 2048 / 2048   & 2048 / 2048  & 2048 / 2048 & 2048 / 2048& 256/256 \\
\hline
 Range of $\lambda$ ~~(\text{Regularization parameter}) & 0-9 / 0-9 &  0-9 / 0-9  & 0-9 / 0-9 & 0-9 / 0-9& 0-500 / 0-500  \\
\hline
Number of Epochs & 200 / 200  & 200 / 200 & 200 / 200 & 200 / 200 & 10/10 \\
\hline
Number of seeds & 5 / 5 & 5 / 5 & 5 / 5 & 5 / 5 & 5/5\\
\hline
Optimizer & Adam / Adam &  Adam / Adam  &  Adam / Adam  & Adam / Adam &  Adam / Adam \\
 \hline
Original (unconstrained) Accuracy & 0.9728 / 0.9728 &  0.6802 / 0.6802 &  0.8427 / 0.8427   & 0.7641 / 0.7641 &  0.7983 / 0.7983\\
 \hline
 Original Bias (i.e., Low-bias Threshold) & 0.2706 / 0.0105   &  0.1257 / 0.1015  &  0.0874 / 0.0165   & 0.4380 / 0.2839 &  0.4429 / 0.2376 \\
 \hline
Threshold for classifier prediction & 0.5 / 0.5 &  0.5 / 0.5  &  0.5 / 0.5  & 0.5 / 0.5 & 0.5 / 0.5 \\
 \hline
 Steps of F-divergence estimator (per classifier update) & 100 / 100   &  100 / 100  &  10 / 1  & 100 / 100 & 3/3\\
 \hline
\end{tabular}}
\caption{Hyperparameters for training process. Each entry represents the hyperparameter w.r.t DP / EO.}
\label{Table: hyperparameters}
\end{center}
\vspace{-15pt}
\end{table*}
The above Proposition uses the convexity property of $\Delta_{DP}(p,q)$. An important consequence of this Proposition is the following: as the fairness constraint (quantified by $\epsilon$) is varied, one achieves different (fairness, accuracy) operating points. The above Proposition shows that the convex hull of these pairs is also achievable (essentially, in our proof we show that to achieve the convex combination of the original tuples, one can construct a new classifier by combining these classifiers). 
We now define the fairness-accuracy receiver operating characteristic (FA-ROC) with respect to demographic parity (resp. equalized odds) as follows:
\begin{align}
\text{FA-ROC}_{\text{DP}}&=\{(\epsilon, a): (\epsilon,a)  \text{ is achievable w.r.t. DP}\}. \nonumber\\
\text{FA-ROC}_{\text{EO}}&=\{(\epsilon,a): (\epsilon,a)  \text{ is achievable w.r.t. EO}\}. \nonumber
\end{align}
The performance of different regularization techniques can be compared by calculating the area under FA-ROC curve (denoted by FA-AUC) as shown in Fig \ref{Fig:auc}(a).  

\noindent \textbf{Low-bias FA-ROC}-- One shortcoming of the FA-ROC is that it will \textbf{mask} the performance of the fair classifier when the original bias of a classifier is small as we explain next. Suppose that we do not impose any fairness constraint, and let $\zeta$ denote the bias of the resulting unconstrained classifier. Then, the \textit{low-bias} FA-ROC consists of all achievable $(\epsilon, a)$ pairs such that $\epsilon\leq \zeta$. Intuitively, since $\zeta$ is the natural bias one would obtain when not imposing any fairness penalty, the non-trivial portion of the trade-off is the one corresponding to $\epsilon \leq \zeta$.  Therefore, we introduce the notion of \textit{low-bias FA-ROC}, and argue that this is a more intuitive and justifiable measure of comparing the performance of different regularization techniques. The low-bias region and corresponding AUC are shown in Fig. \ref{Fig:auc}(a). In Fig. \ref{Fig:auc}(b), we also show the low-bias threshold(s) ($\zeta$) as well as the unconstrained classification accuracy for the four real-world datasets used in our experiments (COMPAS, Adult Income, Law School admissions and CelebA datasets). 
\vspace{-15pt}
\subsection{Experimental Setup}\label{sec:exp_setup} 
\vspace{-1pt}
\subsubsection{Datasets} We consider four real-world datasets and one synthetic dataset in our experiments as described next:  a) \textit{\underline{COMPAS Dataset}}: This dataset consists of data from $N=7,214$ users ($N_{train} = 5,049$, $N_{test} = 2,165$), with 10 features (including age, prior criminal history, charge degree etc.) which are used for predicting the risk of recidivism in the next two years. 
 b) \textit{\underline{Adult Census Dataset (Adult dataset)}}: This dataset includes income related data with 14 features (i.e., age, work class, occupation, education etc.) of $N=45,222$ users ($N_{train} = 32,561$, $N_{test} = 12,661$) to predict whether the income of a person exceeds a threshold (e.g., \$50k) in a year.  c) \textit{\underline{Law School Admissions Dataset (Law School dataset)}}: This dataset includes the admission related data with 7 features (LSAT score, gender, undergraduate GPA etc.) of $N=4,862$ applicants ($N_{train} = 3 ,403$, $N_{test} = 1,459$) to predict the likelihood of passing the bar.  
 d) \textit{\underline{CelebA dataset}}: This high-dimensional image dataset contains $N= 202,599$ ($N_{train} = 162,770$, $N_{validation} = 19,867$, $N_{test} = 19,962$) face images of celebrities with 40 binary attributes, which are cropped and resized to 64 $\times$ 64 pixels images. We use \textit{attractive/not-attractive} as the binary classification label and \textit{gender} as the sensitive attribute. The original CelebA contains training, validation and testing data. We use the CelebA testing data to report test accuracy and fairness measurements. e) \textit{\underline{Moon Dataset}}: This synthetic dataset contains $N=15,000$ examples ($N_{train} = 10,000$, $N_{test} = 5,000$) with two features and one sensitive attribute.

For all the above datasets, our goal will be to build a fair binary classifier for two scenarios:  a) $|Z|=2$, when the sensitive attribute is race (COMPAS/Law School/Adult datasets), i.e., $Z \in \{C, O\}$, where $C=\text{``Caucasian"}$ or $O=\text{``Other race"}$, corresponding to two groups; or when the sensitive attribute is gender (CelebA dataset). We also study another scenario, when $|Z|=4$, when the sensitive attribute(s) are both race and gender (COMPAS/Law School/Adult datasets), i.e., $Z \in \{(C, M), (C, F), (O, M), (O, F), (C, O), (M, F)\}$. Experimental results on multiple sensitive attributes are presented in the Appendix.

\subsubsection{Other Methods on Learning Fair Classifiers} We compare and demonstrate the superiority of our proposed $F$-divergence framework with several existing regularization methods, pre-processing, and post-processing techniques as described next: (1) \textit{KDE based Regularization}: Cho et al.\cite{cho2020fair}: A kernel density estimation (KDE) based fair training, which directly estimates the conditional distribution $p(\hat{Y}|Z)$ (or $p(\hat{Y}|Z,Y)$) using KDE and uses it for fairness regularization. (2) \textit{Correlation based Regularization}: Mary et al. \cite{mary2019fairness}: A correlation coefficient based fair training, which use correlation coefficient $\rho(\hat{Y}, Z)~(or~\rho(\hat{Y}, Z~| Y))$  as the regularization term added in the loss function. (3) \textit{Covariance based Regularization}: Zafar et al. \cite{zafar2017fairness,zafar2017fairness1}: For demographic parity (DP), \cite{zafar2017fairness} uses the covariance between the sensitive attributes and the signed distance from the feature vectors to the decision boundary. For equalized odds (EO), \cite{zafar2017fairness1} uses the covariance between sensitive attributes and the signed distance between the feature vectors of misclassified data and the classifier decision boundary as the regularization term.  (4) \textit{FNR-FPR based Regularization}: Yahav et al. \cite{bechavod2017learning}, which proposes the difference of FNR across different groups and the difference of FPR across different groups as the regularization term. (5) \textit{Pre-processing Method (LFR)}: Zemel et al. \cite{zemel2013learning}, which proposes a pre-processing method to learn a fair representation of the data and then use the fair representation to train the model. (6) \textit{Post-processing Method (ROC) }: Hardt et al. \cite{hardt2016equality}, which proposes a post-processing method to learn a fair predictor from a discriminatory binary predictor by finding the optimal threshold between TPR and FPR for equalized odds (EO) constraints. For all comparison methods, we use authors' original codes or codes for these papers adapted from AI Fairness 360\footnote{https://aif360.mybluemix.net/}.

\begin{figure*}[t]
    \centering
    \includegraphics[scale=0.8]{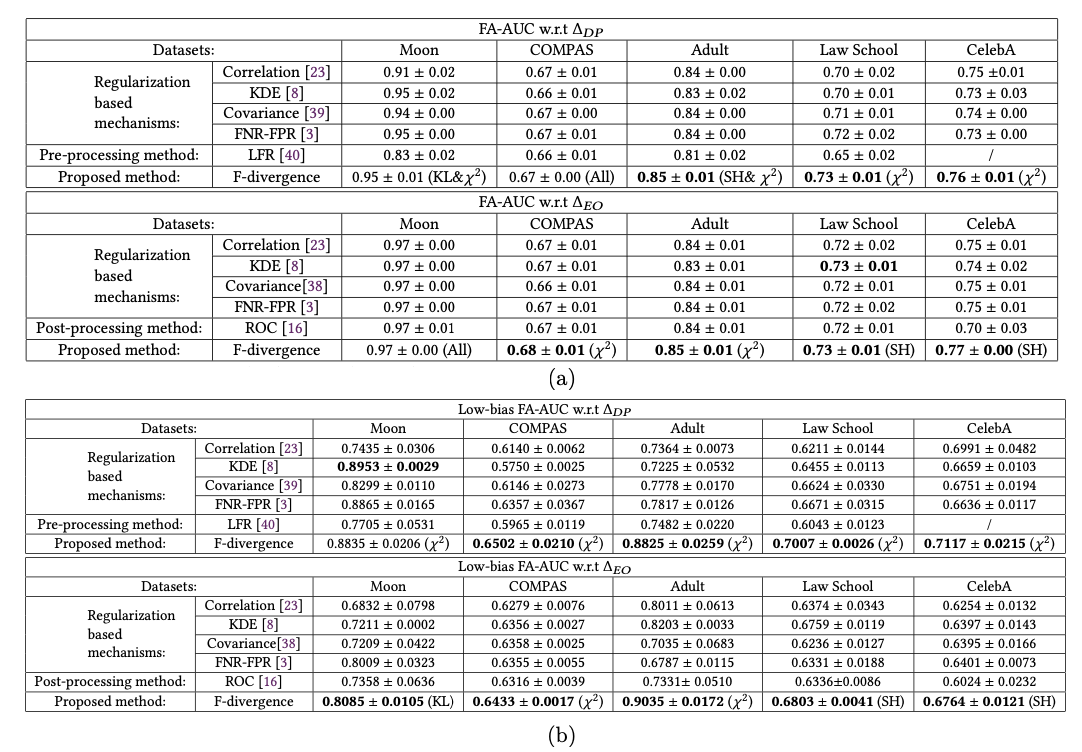}
    \caption{(a) FA-AUC w.r.t $\Delta_{DP}$ (top) / $\Delta_{EO}$ (bottom) and Accuracy; (b) FA-AUC in the \textbf{low-bias} region w.r.t $\Delta_{DP}$ (top) / $\Delta_{EO}$ (bottom) and Accuracy: We compare our best results in Table \ref{Table: f-div} with regularization based methods (correlation \cite{mary2019fairness}, KDE \cite{cho2020fair}, covariance \cite{zafar2017fairness,zafar2017fairness1}, FNR-FPR \cite{bechavod2017learning}), the pre-processing method (LFR \cite{zemel2013learning}) and post-processing method (ROC \cite{hardt2016equality}) method. We show the overall Area-under-the-curve (AUCs) (mean AUC $\pm$ standard deviation) for different techniques under both DP and EO fairness notions.  Bracket All represents that all KL, SH, $\chi^2$ divergence can achieve the value. Bracket SH represents that only SH divergence can achieve the value.}
    \label{fig:overall_table}
    \vspace{-15pt}
\end{figure*}
\begin{figure*}
    \centering
    \includegraphics[scale=0.45]{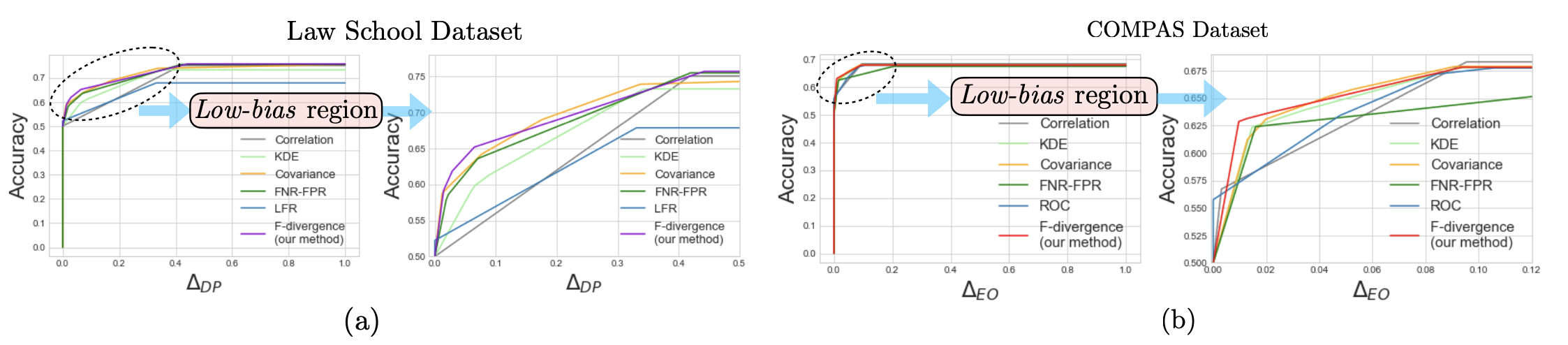}
    \caption{(a) Trade-off between fairness and accuracy (FA-ROC) on the Law School Admission dataset w.r.t \textbf{DP} constraints (left) and corresponding low-bias region FA-ROC (right); (b) Trade-off between fairness and accuracy (FA-ROC) on the COMPAS dataset w.r.t \textbf{EO} constraints (left) and corresponding low-bias region FA-ROC (right). As shown in the figure, for both DP and EO constraints, our method outperforms other regularization based methods. }
    \label{fig:law_dp}
    \vspace{-15pt}
\end{figure*}

\subsubsection{Model Architectures, Training Methodology and Hyperparameters}: For low/medium dimensional datasets, namely Moon, COMPAS, Adult and Law School, we consider neural network classifiers with three fully connected layers, where each hidden layer has $200$ nodes and is followed by an $SeLU$ \cite{klambauer2017self} non-linear layer. According to the experimental results, three fully connected layers are sufficient to obtain state-of-art accuracy; we report the original accuracy (without fairness constraints) in Table \ref{Table: hyperparameters}. We obtained the best accuracy with SeLU activation (in comparison to ReLU, Sigmoid activations). 
For the proposed F-divergence based fair training, the F-divergence estimator is modeled using two-layer neural networks, each with $5$ hidden nodes and a sigmoid non-linearity (see additional discussion on the choice of the model architecture in the Appendix).
For high dimensional image classification dataset (CelebA), we use ResNet 18 \cite{he2016deep} as the classifier for predicting attractive/un-attractive. For the F-divergence estimator, we use three linear layers, each with $10$ hidden nodes and a sigmoid non-linearity. Note that this architecture is sufficient for CelebA dataset since we only take outputs of the classifier and sensitive attributes as the input of the F-divergence estimator.

For all regularization based training methods, we use the following training loss function: $\mathcal{L}_{\text{Error}}+ \lambda \mathcal{R}_{\text{Fairness}}$, where $\mathcal{L}_{\text{Error}}$ denotes the binary cross-entropy loss (for the classification error) together with the fairness related regularization term $\mathcal{R}_{\text{Fairness}}$. The trade-off parameter $\lambda$ is used to adjust the proportion of classification loss and fairness constraints; varying $\lambda$ give us the set of points reflecting the trade-off between fairness and accuracy. To compare with pre / post - processing methods, we vary the classification threshold(s) of the corresponding algorithms to form a FA-ROC. By constructing the convex hull w.r.t these points, we form the FA-ROC as shown in Fig \ref{Fig:auc}. For each $\lambda$, we train over 5 runs with different random seeds. To obtain the Low-bias FA-ROC, we set $\lambda = 0$ (i.e., no fairness constraint), and find the corresponding bias of the learned classifier. This yields the low-bias threshold $\zeta$, and the low-bias FA-ROC is the convex hull of all points for which the bias is no more than $\zeta$ ($\zeta$ w.r.t DP / EO are shown in Table \ref{Table: hyperparameters}). For all methods, the models were trained using the Adam optimizer. For fair comparison, we use same training steps and the same classification model for other mechanisms. We report the accuracy and fair measurements on the test dataset. All the training hyperparameters are summarized in Table \ref{Table: hyperparameters}. 
For a fair comparison of different regularization techniques (including the $F$-divergence techniques and prior works) pre- and post-processing methodologies, we compare all of the techniques by measuring fairness using the well accepted notion of TV distance ($\Delta_{DP}$ as defined in \eqref{def:deltaDP} for DP or $\Delta_{EO}$ \eqref{def:deltaEO} for EO), which are the same metrics as defined in previous works.


\subsection{\ul{Comparison of  $F$-divergence Regularization with other approaches for learning fair classifiers}} \label{sec:exp_results}
In Fig \ref{fig:overall_table} (a), we report the overall FA-AUC for both DP and EO fairness notions for the six other existing methods and the proposed $F$-divergence approach. The numbers in the table are the mean AUC $\pm$ standard deviation over $5$ independent trials.  We notice that our proposed method achieves better trade-off than other compared methods for all the real world datasets. We also report the low-bias FA-AUC for all the datasets and the different regularization techniques in Fig \ref{fig:overall_table} (b). In contrast to the overall FA-AUC (where the performance improvement given by $F$-divergence is modest), when we zoom in the low-bias region, we see a significant performance improvement in the fairness-accuracy tradeoffs. As an example, for the Adult dataset with the DP fairness objective, our proposed method achieves \textbf{12.89\%} higher low-bias FA-AUC than other compared state-of-art mechanisms (the gain for EO fairness objective is \textbf{10.14\%}). Our method also performs well on high dimensional dataset (CelebA) where we achieve a gain of \textbf{2.67\% }/ \textbf{5.67\%} w.r.t the EO constraint for FA-AUC / Low-bias FA-AUC. We show the plots of FA-ROC in Fig \ref{fig:law_dp} on Law School and COMPAS datasets (similar figures for other datasets are provided in the Appendix). Another interesting observation is that the optimal choice of $F$-divergence regularization is dependent on both the dataset as well as the notion of fairness. For instance, Pearson $\chi^2$ divergence regularization yields the best trade-off than other $F$-divergence based methods w.r.t DP on all real world datasets. For the equalized odds (EO) notion,  SH divergence regularization shows better performance on  both COMPAS and Law School datasets while $\chi^2$ divergence achieves better results on Adult dataset.

\vspace{-5pt}
\section{Concluding Remarks}
\vspace{-3pt}
In this work, we introduced a general min-max $F$-divergence regularization framework for fair classifiers, which is readily adaptable for high-dimensional problems and compatible with gradient based optimization methods. In contrast to existing regularization methods such as correlation/covariance/TPR/FPR based methods, $F$-divergence notions for quantifying fairness/bias are stronger notions to capture dependence between classifier decisions and sensitive attributes and provide stronger fairness guarantees.

We also proposed the notion of Fairness-Accuracy ROC (FA-ROC) and a corresponding low-bias FA-ROC, which we argue as the correct approach to compare different mechanisms for learning fair classifiers and quantifing the tradeoff between fairness and accuracy. Through an extensive set of experiments on four real-world datasets (COMPAS, Adult Income, Law School Admissions and CelebA), we demonstrated the improvement in the fairness-accuracy tradeoffs compared to prior works on regularization as well as pre- and post-processing methods. 
\vspace{-10pt}
\bibliographystyle{unsrt}
\bibliography{main}

\appendices
\section{Implementation details}
\subsection{Additional Description of Datasets and Pre-processing}

In this Section, we present additional details regarding the synthetic (Moon) dataset, as well as some pre-processing steps we followed for the Law Admissions dataset. 

\textbf{Synthetic Moon Dataset}: For the synthetic Moon dataset, we use scikit-learn \cite{pedregosa2011scikit} -- $make\_moons$ function to generate two interleaving half circles with $P(Y=1|Z=1)\approx 0.35$ and $P(Y=1|Z=0)\approx 0.65$ while $P(Y=1)\approx0.5$, which follows the same procedure in KDE based mechanism \cite{cho2020fair}. 

\textbf{Law Admissions Dataset}: The law school dataset contains $ 21,791 $ numbers of tabular data from law students, where $19,360$ data have positive label and $2,431$ users contain negative label, which is an extremely unbalanced dataset. For this kind of extremely imbalanced dataset, a classifier will attend to predict more positive label leading to a skewed classification \cite{ yen2006under}. To overcome this situation, we balance the dataset by under-sampling the majority group w.r.t values of label (Y = 0 / Y = 1). Specifically, we randomly sample $2,431$ data points from majority group (positive label) and combine with the minority group to form a balanced law school admission dataset. 

\subsection{Choice of activation functions and Architecture for F-divergence estimator} Our binary classifier architecture consists of three fully connected layers, where each hidden layer has $200$ nodes and is followed by a Scaled exponential Linear Unit (SeLU) non linearity, defined as $\text{SeLU}(x) = \text{scale}\times \max(0,x) + \min (0, \alpha\times (e^{x}-1)$, where $\text{scale}= 1.0507$ and $\alpha = 1.6732$.  The last layer is a sigmoid non-linear layer which outputs the probabilities for the two classes. We compare other activation methods with SeLU activation function on COMPAS dataset without fairness constraints as shown in Table \ref{tab:acti_cls}. We find that SeLU activation function yields higher accuracy ($0.68$) while preserving fairness constraints (bias $\approx 0.12$).
\begin{table}[h]
    \centering
    \begin{tabular}{|c|c|c|}
        \hline
         & Unconstrained Accuracy& $\Delta_{DP}$  \\
        \hline    
         ReLU  & 0.6660 $\pm$ 0.0041 & 0.0991 $\pm$ 0.0038\\
        \hline
        Sigmoid &0.6769 $\pm$ 0.0030 & 0.1302 $\pm$ 0.0053 \\
        \hline
        SeLU & 0.6802 $\pm$ 0.0056 & 0.1257 $\pm$ 0.0083\\
        \hline
    \end{tabular}
    \caption{Comparison of activation functions used for the classifier on the COMPAS dataset. We compare the SeLU activation function with ReLU and Sigmoid activation function. We can observe that SeLU achieves the best trade-off than other two activation functions.}
    \label{tab:acti_cls}
\end{table}

\begin{table}[h]
    \centering
    \begin{tabular}{|c|c|c|}
        \hline
         & Constrained Accuracy& $\Delta_{DP}$  \\
        \hline    
         ReLU  & 0.6666 $\pm$ 0.0370   & 0.1077 $\pm$ 0.0426\\
        \hline
        SeLU &0.6008  $\pm$ 0.0119 & 0.0238 $\pm$ 0.0164\\
        \hline
        Sigmoid &0.6056 $\pm$ 0.0111 & 0.0176 $\pm$ 0.0245 \\
        \hline
    \end{tabular}
    \caption{Comparison of activation functions for the $F$-divergence estimator on the COMPAS dataset, where we enforce largest fairness constraints ($\lambda = 9$) to the classifier during the training process and output corresponding test accuracy and $\Delta_{DP}$.  We compare the Sigmoid activation function with ReLU and SeLU activation function. We can observe that Sigmoid gives the best trade-off than other two activation functions.}
    \label{tab:acti_div}
\end{table}

We next present the details regarding the architecture of the $F$-divergence estimator. We train the NN based $F$-divergence estimator with two hidden layers each followed by a sigmoid non-linearity. Similar as above, in Table \ref{tab:acti_div}, we compare sigmoid function with SeLU and ReLU activation functions on COMPAS dataset while imposing the largest fairness constraints ($\lambda=9$) with respect to demographic parity. We find that sigmoid function gives the best trade-off between fairness and accuracy.



\subsection{Other Approaches for Estimating $F$-divergence}
In this section, we discuss other possible alternatives to estimate $F$-divergence. We also perform a comprehensive study on the choice of $F$-divergence estimator in our experiments. 

\noindent \textbf{a) Conventional representation of $F$-divergence (CON)}
Another alternative is to directly estimate the probability distributions $\hat{\pi}_{i}$ and $\hat{\pi}_{j}$, and plug them to get an estimate of $F$-divergence term as follows:
 \begin{align}
     D_f(\pi_i\parallel \pi_j) \approx \sum_{u\in \{0,1\}}      \hat{\pi}_j(u)\times f\left(\frac{\hat{\pi}_i(u)}{\hat{\pi}_j(u)}\right)
 \end{align}
which can be used as the fairness regularization term. 

\noindent \textbf{b) Density ratio based estimation (DRE) of $F$-divergence}
Another approach is to estimate the density ratio (in our case, a probability ratio) $r(u)= \pi_i(u)/\pi_j (u)$ from the i.i.d. samples from the groups $i,j$ borrowing techniques from density ratio estimation (DRE) \cite{sugiyama2012density}. With an estimate $\hat{r}(u)$ for the probability ratio, one can then estimate $F$-divergence as follows: 
\begin{align}
D_f(\pi_i\parallel \pi_j) \hspace{-1pt}=\hspace{-1pt} E_{X\sim \pi_j} f(r(X)) \hspace{-1pt}\approx\hspace{-1pt} \frac{1}{M} \hspace{-2pt}\sum_{m=1}^{M}f(\hat{r}(x_{j}^{(m)}))
\end{align} 
where $\{x_{j}^{(m)}\}$ denote i.i.d. samples drawn from the distribution $\pi_{j}$. In Section \ref{sec:exp_results}, we perform a comprehensive comparison of various methods for estimating $F$-divergence and their associated impact on the tradeoffs between accuracy and fairness.  

\section{ADDITIONAL EXPERIMENTAL RESULTS}
In this Section, we first provide additional experimental results (FA-ROC and Low-bias FA-ROC) about the tradeoff between accuracy and fairness on Moon / Adult / CelebA datasets. Next, we compare our proposed method with other regularization based methods (FNR-FPR \cite{bechavod2017learning}, Correlation \cite{mary2019fairness} and Covariance \cite{zafar2017fairness, zafar2017fairness1}) w.r.t the groupwise AUC for DP and EO constraints.
\begin{figure*}
    \centering
    \includegraphics[scale=0.45]{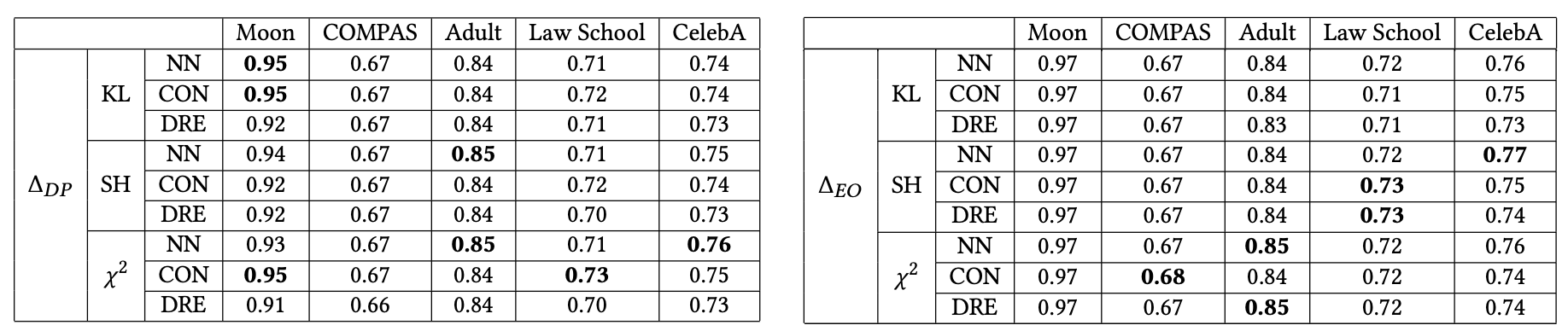}
    \caption{Evaluation of the $F$-divergence based framework (via FA-AUC) w.r.t various $F$-divergence Estimators (NN, CON, DRE) and different $F$-divergences (KL, SH, $\chi^2$) for the DP (left) / EO (right) constraint. We can observe that NN based $F$-divergence estimator consistently leads to higher area under the FA-ROC curve.}
    \vspace{-15pt}
    \label{Table: f-div}
\end{figure*}
\begin{figure*}[t]
\centering
\includegraphics[scale=0.53]{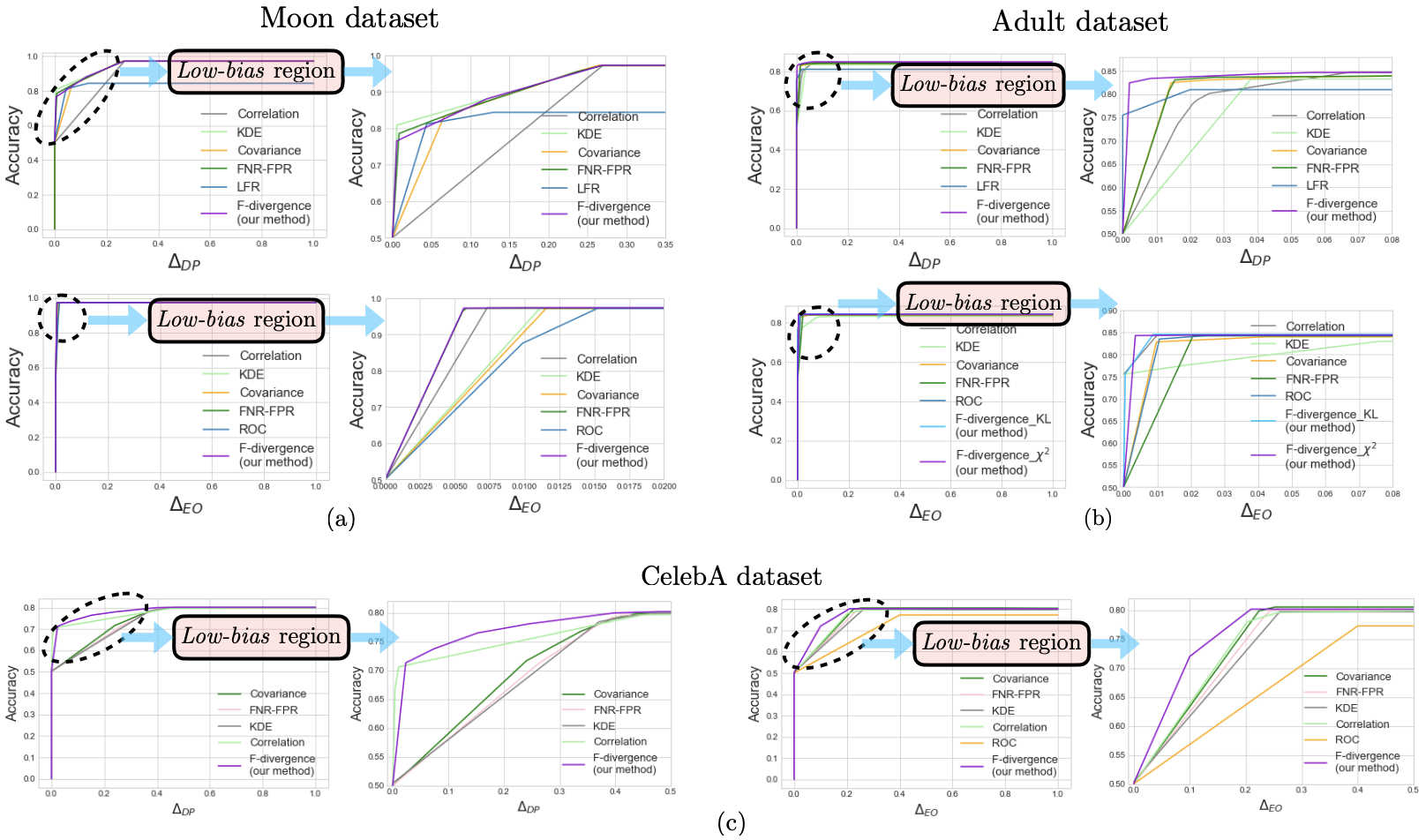}
\caption{Trade-off between fairness-accuracy (FA-ROC) for DP and EO on the Moon dataset (a) / Adult dataset (b) / CelebA dataset (c). We also show the low-bias region (Low-bias FA-ROC) in these figures. For both DP and EO constraints, we observe that our proposed method consistently outperforms other regularization based methods (Correlation \cite{mary2019fairness}, KDE \cite{cho2020fair}, Covariance \cite{zafar2017fairness,zafar2017fairness1}, FNR-FPR \cite{bechavod2017learning}), the pre-processing method (LFR \cite{zemel2013learning}) and post-processing method (ROC \cite{hardt2016equality}).}
\label{dp-moon-divergence}
\end{figure*}

\subsection{\underline{Evaluation of different $F$-divergence Estimators}} To validate the $F$-divergence regularization framework, we first conduct experiments comparing different $F$-divergence estimation techniques. Specifically, we test the three methods discussed in Section \ref{Sec:f-div} (variational representation (NN), convention representation  (CON) and DRE method (DRE)) for three types of $F$-divergence measures:  KL divergence (KL), Squared Hellinger distance (SH) and Pearson $\chi^2$ divergence ($\chi^2$).  In Fig \ref{Table: f-div}, we show the performance of trade-off between fairness and accuracy w.r.t. different $F$-divergence measures under different estimation techniques for the five datasets. We first observe that the $F$-divergence estimator leads to consistent results for all five datasets. We also observe that NN and CON methods for $F$-divergence estimation can be slightly better than the DRE method. Therefore, we will primarily use NN based $F$-divergence estimator in the following experiments unless specified otherwise. Our second interesting observation is that the optimal choice of the $F$-divergence measures (KL vs. SH vs. $\chi^{2}$) is dependent on the dataset. For instance, for demographic parity, $\chi^2$-divergence regularization outperforms other two $F$-divergences on the Law School dataset.

\subsection{\ul{Fairness-accuracy tradeoffs on Moon, Adult and CelebA datasets}}

In this Section, we show the plots for Moon, Adult and CelebA datasets, which are shown in Fig. \ref{dp-moon-divergence}. For Moon dataset, our proposed method with KL divergence is close to the state-of-the-art compared methods w.r.t DP fairness objectives. For EO, three of our proposed technique has better trade-off between fairness and accuracy than others. For the Adult dataset, we have a similar observation that our proposed methods outperform other compared mechanisms (the gain for DP is \textbf{12.89\%}; the gain for EO is \textbf{10.14\%}). We can observe similar trends on CelebA dataset, where our proposed method consistently achieves better results than others (achieves a gain of \textbf{2.67\% }/ \textbf{5.67\%} w.r.t the EO constraint for FA-AUC / Low-bias FA-AUC . ).

\begin{figure*}[t]
\centering
\includegraphics[scale=0.2]{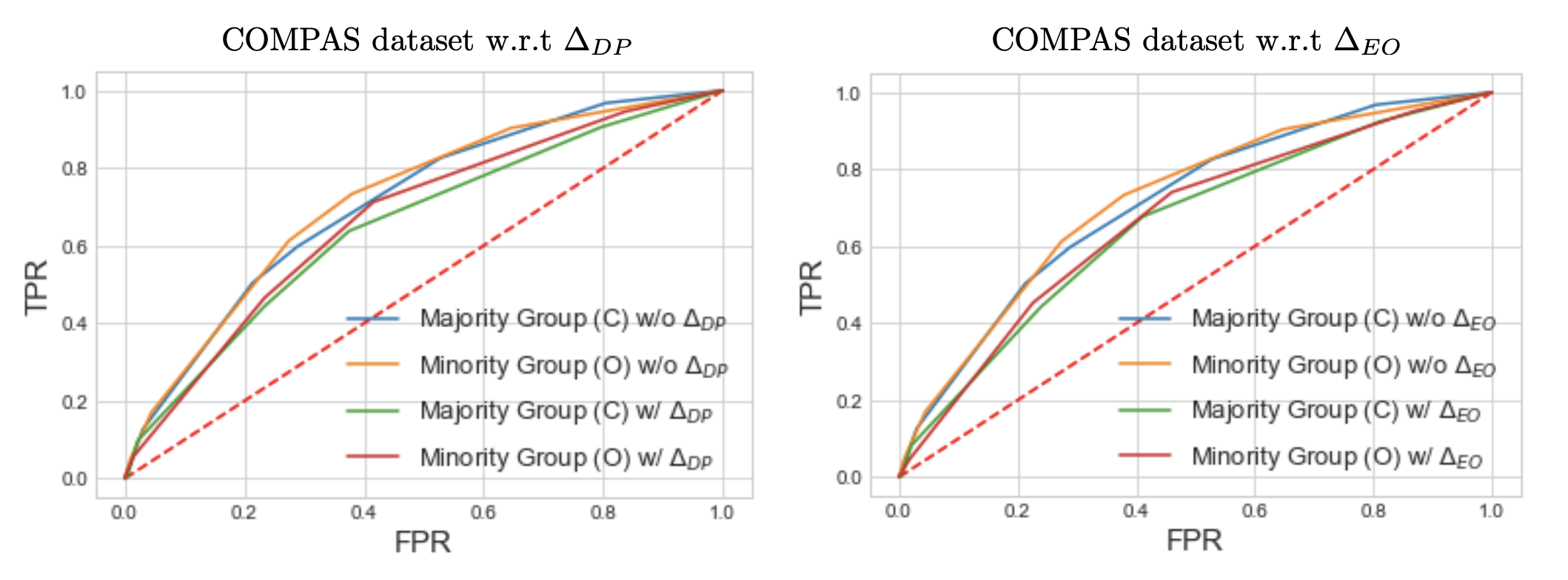}
\caption{Groupwise ROC w.r.t. TPR and FPR on the COMPAS dataset for DP constraints (left) and for EO constraints (right). For COMPAS dataset, the majority group is "Caucasian" and the minority group corresponds to "Other Race". }
\label{Fig: groupwise}
\end{figure*}

\begin{figure*}[h]
\centering
\includegraphics[scale=0.50]{supp_table.png}
\caption{Comparison of groupwise AUC w.r.t TPR and FPR with other regularization based methods (FNR-FPR \cite{bechavod2017learning}, Correlation \cite{mary2019fairness} and Covariance \cite{zafar2017fairness,zafar2017fairness1} based method ) for each population group w.r.t DP (a) / EO (b). As shown above, $F$-divergence framework has mild reductions on each group while ensuring fairness constraints for all three real world datasets.}

\label{fig:supp_table}.
\end{figure*}
\subsection{\underline{Impact of Fairness Constraints on Groupwise TPR/FPR}} We now study the impact of fairness constraints on the groupwise  ROC and AUC. Specifically, for each population sub-group (e.g., for COMPAS/Adult/Law-School datasets, the population was divided into "C" (Caucasian) and "O" (other race)),   we calculate the true positive rate (TPR) and the false positive rate (FPR) of the fair classifier and show the resulting ROC in Fig \ref{Fig: groupwise}. 
To explore the impact of fairness constraints, we compare the ROC of a classifier without fairness constraints ($\lambda=0$) to the classifier trained with largest fairness constraint (i.e., highest regularization penalty, e.g., $\lambda = 9$ for COMPAS/Adult/Law-School datasets) in our experiments. We then calculate the area under ROC to get the quantitative results, which are shown in Table \ref{Tab:groupwise}. We observe that for both minority and majority population groups, the $F$-divergence framework only leads to a mild reduction in the AUC with fairness constraints, when compared with the AUC without fairness constraints. 
\begin{table*}[t]
\begin{center}
\scalebox{1}{
\begin{tabular}{|c|c|c|c|c|c|c|c|c|}
\hline
 & \multicolumn{2}{|c|}{Majority Group} & \multicolumn{2}{|c|}{Minority Group}& \multicolumn{2}{|c|}{Majority Group} & \multicolumn{2}{|c|}{Minority Group} \\
 \cline{2-9}
 &  w/o DP & w/ DP &  w/o DP & w/ DP &  w/o EO & w/ EO 
&  w/o EO & w/ EO \\
\hline
COMPAS & 0.72& 0.67  & 0.71  & 0.65 &0.72 & 0.66  & 0.71 & 0.66 \\
\hline
Adult  & 0.90 & 0.88 & 0.88  & 0.85  &0.90 & 0.86 & 0.88  & 0.85 \\
\hline
Law School & 0.81& 0.80 & 0.81  & 0.80 & 0.81 & 0.76 & 0.81  & 0.75 \\
\hline
\end{tabular}}
\caption{Groupwise AUC w.r.t DP (left) / EO (right) for each population group on real world datasets with (or without fairness constraints). As shown in the Table, groupwise AUC perserves high utility while adding fair constraints (DP or EO).  }
\label{Tab:groupwise}
\end{center}
\end{table*}
Next, We study the impact of our proposed methods on the groupwise ROC and AUC compared with other mechanisms (FNR-FPR \cite{bechavod2017learning}, Correlation \cite{mary2019fairness} and Corvariance \cite{zafar2017fairness, zafar2017fairness1} based method). For each population group (C or O), we construct the ROC with respect to TPR and FPR of the classifier (with or without fairness constraints) and then calculate the AUC in Fig \ref{fig:supp_table}. $F$-divergence based framework consistently leads to a slightly reduction in the AUC with fairness objectives compared with the classifier without fairness constraints. As an example, for Law School dataset, the reduction of enforcing DP fairness constraints for each population group is only \textbf{0.01}. Similar as Law school dataset, the maximum reduction among each group in Adult dataset is \textbf{0.03} for DP and \textbf{0.04} for EO.

\begin{table*}[t]
\begin{center}
\scalebox{0.9}{
\begin{tabular}{|c|c|c|c|}
\hline
\multicolumn{4}{|c|}{FA-AUC}\\
\hline
  & COMPAS  & Adult  & Law School \\
\hline
 Correlation \cite{mary2019fairness}&  0.59  $\pm$ 0.04 & 0.83 $\pm$ 0.03 & 0.63 $\pm$ 0.03\\
\hline 
KDE\cite{cho2020fair} & 0.67 $\pm$ 0.00 &  0.81 $\pm$ 0.02 & 0.72 $\pm$ 0.01 \\
\hline
 Covariance\cite{zafar2017fairness} &0.67  $\pm$ 0.01 & 0.83 $\pm$ 0.01 &  0.72 $\pm$ 0.01 \\
\hline
 FNR-FPR \cite{bechavod2017learning}& 0.67 $\pm$ 0.00 & 0.80 $\pm$ 0.02 & 0.72 $\pm$ 0.03\\
 \hline
Proposed method ($\chi^2$) & 0.67 $\pm$ 0.00 & \textbf{0.84 $\pm$ 0.01} &\textbf{0.73 $\pm$ 0.01}\\
\hline
\hline
\multicolumn{4}{|c|}{Low-bias FA-AUC}\\
\hline
   & COMPAS  & Adult  & Law School \\
\hline
 Correlation \cite{mary2019fairness} & 0.5961 $\pm$ 0.0299 & 0.6688 $\pm$ 0.0349 & 0.6182 $\pm$ 0.0396 \\
\hline 
KDE\cite{cho2020fair} & 0.6120 $\pm$ 0.0449 &  0.7652 $\pm$ 0.0380 & 0.6730 $\pm$ 0.0168\\
\hline
Covariance\cite{zafar2017fairness} & 0.6192 $\pm$ 0.0410 & \textbf{0.8092 $\pm$ 0.0759} &  0.6863 $\pm$ 0.0374 \\
\hline
 FNR-FPR \cite{bechavod2017learning} & 0.5974 $\pm$ 0.0073 & 0.7071 $\pm$ 0.0048 & 0.6863 $\pm$ 0.0192 \\
 \hline
Proposed method ($\chi^2$) & \textbf{0.6328  $\pm$ 0.0086} & 0.8067 $\pm$ 0.0226 &\textbf{0.7000 $\pm$ 0.0170}\\
\hline
\end{tabular}}
\caption{FA-AUC and Low-bias FA-AUC w.r.t multiple sensitive attributes (gender and race). We can observe that $F$-divergence regularization consistently yields the better AUC in most scenarios.}
\label{Table: multi sensi}
\end{center}
\end{table*}

\subsection{\underline{Results for Multiple Sensitive Attributes}.}
We next present results for the case when more than one attributes can be sensitive (specifically, both race and gender) for COMPAS, Adult and Law School datasets, which leads to four population groups. According to Eq \eqref{fair_regu}, we calculate the sum of $\Delta_{DP}$ on each pair of groups and then add them in the classification loss function. For proposed method, we use $\chi^2$ divergence based NN model for COMPAS and Adult datasets. For Law School dataset, we apply $\chi^2$ divergence based regularization which is estimated using the conventional (CON) method. As shown in Table \ref{Table: multi sensi}, our proposed method achieves better performance than other existing regularization techniques, i.e., for the COMPAS dataset with DP fairness constraints, our proposed method achieves \textbf{2.20\%} higher low-bias FA-AUC than other compared state-of-art techniques (the gain for Law School dataset is \textbf{2.00\%}).

\end{document}